\documentclass[runningheads]{llncs}
\usepackage{graphicx}
\usepackage{amsmath,amssymb} 
\usepackage{color}
\usepackage[width=122mm,left=12mm,paperwidth=146mm,height=193mm,top=12mm,paperheight=217mm]{geometry}
\usepackage{subfigure}
\usepackage{multirow}
\usepackage{color}
\usepackage{algorithm}
\usepackage{algpseudocode}

\graphicspath{{./figure/}{./figure/draw/}}

\def\x{{\mathbf x}}
\def\a{{\mathbf a}}
\def\u{{\mathbf u}}
\def\z{{\mathbf z}}
\def\v{{\mathbf v}}

\def\I{{\mathbf 1}}

\def\H{{\mathbf H}}
\def\Y{{\mathbf Y}}
\def\L{{\mathbf L}}
\def\D{{\mathbf D}}
\def\E{{\mathbf E}}
\def\Z{{\mathbf Z}}

\def\M{{\mathbf M}}

\def\I{{\mathbf I}}
\def\1{{\mathbf 1}}
\def\S{{\mathbf S}}

\def\X{{\mathbf X}}
\def\Y{{\mathbf Y}}
\def\A{{\mathbf A}}
\def\B{{\mathbf B}}

\def\U{{\mathbf U}}
\def\E{{\mathbf E}}
\def\M{{\mathbf M}}
\def\H{{\mathbf H}}
\def\I{{\mathbf I}}

\def\R{{\mathbb R}}
\def\L{{\mathbf L}}
\def\O{{\mathcal O}}

\newcommand\norm[1]{\left\lVert#1\right\rVert}
\begin{document}
\pagestyle{headings}
\mainmatter

\title{Binary Hashing with Semidefinite Relaxation and Augmented Lagrangian} 

\titlerunning{\textit{Binary Hashing with Semidefinite Relaxation and Augmented Lagrangian}}

\authorrunning{\textit{Thanh-Toan Do, Anh-Dzung Doan, Duc-Thanh Nguyen, Ngai-Man Cheung}}

\author{Thanh-Toan Do, Anh-Dzung Doan, Duc-Thanh Nguyen, Ngai-Man Cheung}


\institute{Singapore University of Technology and Design \\ {\tt\small \{thanhtoan\_do, dung\_doan, ducthanh\_nguyen, ngaiman\_cheung\}@sutd.edu.sg}}

\maketitle

\begin{abstract}
This paper proposes two approaches for inferencing binary codes in two-step (supervised, unsupervised) hashing. 
We first introduce an unified formulation for both supervised and unsupervised hashing.
Then, we cast the learning of one bit as a Binary Quadratic Problem (BQP).  We propose two approaches to solve BQP. In the first approach, we relax BQP as a semidefinite programming problem which its global optimum can be achieved. We theoretically prove that the objective value of the  binary solution achieved by this approach is well bounded. 
In the second approach, we propose  an augmented Lagrangian based approach to solve BQP directly without relaxing the binary constraint. Experimental results on three benchmark datasets show that our proposed methods compare favorably with the state of the art.
\end{abstract}

\section{Introduction}
\vspace{-0.2cm}
Hashing methods construct a set of hash functions that map the original high dimensional data into  low dimensional binary data. The resulted binary vectors allow efficient storage and fast searching, making hashing as an attractive approach for large scale visual search~\cite{Grauman_review,DBLP:journals/corr/WangSSJ14}.

Existing hashing methods can be categorized as data-independent and data-dependent schemes.
 Data-independent hashing methods~\cite{lsh_vldb09,KLSH_iccv09,KLSH_nips09,DBLP:journals/pami/KulisJG09} rely on random projections for constructing hash functions. Data-dependent hashing methods use available training data for learning hash functions in unsupervised or supervised way. 
Unsupervised hashing methods, e.g. Spectral Hashing~\cite{DBLP:conf/nips/WeissTF08}, Iterative Quantization (ITQ) \cite{DBLP:conf/cvpr/GongL11}, K-means Hashing \cite{DBLP:conf/cvpr/HeWS13}, Spherical Hashing \cite{CVPR12:SphericalHashing}, Non-negative Matrix Factorization (NMF) hashing~\cite{Mukherjee_2015_ICCV}, try to preserve the distance similarity of samples. 
Supervised hashing methods, e.g. Minimal Loss Hashing \cite{DBLP:conf/icml/NorouziF11}, ITQ-CCA \cite{DBLP:conf/cvpr/GongL11},  Binary Reconstructive Embedding \cite{Kulis_learningto}, KSH~\cite{CVPR12:Hashing}, Two-Step Hashing~\cite{DBLP:conf/iccv/LinSSH13}, FastHash \cite{CVPR2014Lin}, try to preserve the label similarity of samples. 

Most aforementioned hashing methods follow two general steps for computing binary codes. The first step is to define hash functions together with a specific loss function. Usually, the hash functions take the linear form~\cite{DBLP:conf/cvpr/GongL11,DBLP:conf/icml/NorouziF11,lsh_vldb09} or non-linear (e.g. kernel) form~\cite{KLSH_iccv09,Kulis_learningto,CVPR12:Hashing}. The loss functions are typically defined by minimizing the difference between Hamming affinity (or distance) of data pairs and the ground truth~\cite{CVPR12:Hashing,DBLP:conf/iccv/LinSSH13,Kulis_learningto,CVPR2014Lin}. 
The second step is to solve hash function parameters by minimizing the loss function under  the binary constraint on the codes. 
The coupling of the hash function and the binary constraint often results in a highly non-convex optimization which is very challenging to solve. Furthermore, because the hash functions vary for different methods, different optimization techniques are needed for each of them.
\vspace{-0.3cm} 
\subsection{Related work}
Our work is inspired by a few recent supervised hashing methods~\cite{DBLP:conf/iccv/LinSSH13,CVPR2014Lin} 
 and unsupervised hashing method~\cite{Mukherjee_2015_ICCV} which rely on two-step approach to reduce the complexity of the coupled problem and to make the flexibility in using of different types of hash functions. In particular, those works decompose the learning of hash functions under binary constraint into two steps: the binary code inference step and the hash function learning step. The most difficult step is binary code inference which is NP-hard problem. After getting the binary codes, the hash function learning step becomes a classical binary classifier learning. Hence, it allows the using of various types of hash functions, i.e., linear SVM~\cite{Mukherjee_2015_ICCV}, kernel SVM~\cite{DBLP:conf/iccv/LinSSH13}, decision tree~\cite{CVPR2014Lin}.

In order to infer binary codes, in~\cite{DBLP:conf/iccv/LinSSH13,CVPR2014Lin}, the authors form the learning of one bit of binary code as a binary quadratic problem and using non-linear optimization~\cite{DBLP:conf/iccv/LinSSH13} or Graphcut~\cite{CVPR2014Lin} for solving. 
In~\cite{Mukherjee_2015_ICCV}, the authors solve the binary code inference using non-linear optimization approach or non-negative matrix factorization approach. We will brief the approaches in~\cite{DBLP:conf/iccv/LinSSH13,CVPR2014Lin,Mukherjee_2015_ICCV} 
  when comparing to our methods in section~\ref{subsec:diff}. 

Although different methods are proposed for inferencing the binary code, the disadvantage of those methods~\cite{DBLP:conf/iccv/LinSSH13,Mukherjee_2015_ICCV} 
 is that in order to overcome the hardness of the binary constraint on codes, they solve the relaxed problem, i.e., relaxing the binary constraint to continuous constraint. 
This may decrease the code quality and incurs some performance penalty. 
Furthermore, those works have not theoretically investigated the quality of the relaxed solution.

\vspace{-0.3cm} 
\subsection{Contribution}
Instead of considering separate formulations for supervised hashing and unsupervised hashing, we first present an unified formulation for both. 
Our  main contributions are that we propose two approaches for inferencing binary codes. 
In the first approach, we cast the learning of one bit of the binary code as a Semidefinite Programming (SDP) problem which its global optimum can be achieved. After using a randomized rounding procedure for converting the solution of SDP to the binary solution, we theoretically prove that the objective value of the resulted binary solution is well bounded, i.e., it is not arbitrarily far from the global optimum objective value of the original problem. It is worth noting that although semidefinite relaxation has been applied to several computer vision problems such as image segmentation, image restoration~\cite{DBLP:journals/pami/KeuchelSSC03,DBLP:journals/corr/WangSH14a}, to the best of our knowledge, our work is the first one that applies semidefinite relaxation to the binary hashing problem. 
In the second approach, we propose to use Augmented Lagrangian (AL) for directly solving the binary code inference problem without relaxing the binary constraint.   
One important step in the AL is initialization~\cite{Nocedal06}. 
In this work, we careful derive an initialization to achieve a good feasible starting point. For both SDP and AL approaches, their memory and computational complexity are also analyzed.

The remaining of this paper is organized as follows. Section~\ref{sec:proposed} 
presents proposed approaches for binary code inference. Section~\ref{sec:exp} evaluates and compares proposed approaches to the state of the art. Section~\ref{sec:concl} concludes the paper.  

\vspace{-0.4cm} 
\section{Proposed methods}
\label{sec:proposed}

\subsection{Unified formulation for similarity preserving unsupervised / supervised hashing}
Let $\X \in \R^{D \times n}$ be matrix of $n$ samples; $\S = \{s_{ij}\} \in \R^{n\times n}$ be symmetric pairwise similarity matrix, i.e., pairwise distance matrix for unsupervised hashing or pairwise label matrix for supervised hashing; $\Z = \{z_{ij}\} \in \{-1,1\}^{L \times n}$ be binary code matrix of $\X$, where $L$ is code length; each column of $\Z$ is binary code of one sample; $\D = \{d_{ij}\}\in \R^{n\times n}$, 
where $d_{ij}$ is Hamming distance between samples $i$ and $j$, i.e., columns $i$ and $j$ of $\Z$; we have $0\le d_{ij} \le L$. 
We target to learn the binary code $\Z$ such that 
the similarity matrix in original space is directly preserved through Hamming distance in Hamming space. 
In a natural means, we learn the binary code matrix $\Z$ by solving the following binary constrained least-squares objective function
\begin{equation}
\min_{\Z \in \{-1,1\}^{L \times n}} \norm{\frac{1}{L}\D - \frac{1}{c}\S}^{2} 
\label{eq:dis-preserving}
\end{equation}

In (\ref{eq:dis-preserving}), $c$ is a constant. The scale factors $\frac{1}{L}$ and $\frac{1}{c}$ are to make $\D$ and $\S$ same scale, i.e., belonging to the interval $[0,1]$, when doing least-squares fitting. 
For unsupervised hashing, any distance function can be used for computing $\S$. 
In this work, we consider the squared Euclidean distance which is widely used in nearest neighbor search. By assuming that the samples are normalized to have unit $l_2$ norm, we have $0 \le s_{ij} \le 4$. In this case, the constant $c$ equals to 4. For supervised hashing, we define $s_{ij} = 0$ if samples $i$ and $j$ are same class. Otherwise, $s_{ij} = 1$. In this case, the constant $c$ equals to 1.

In~\cite{CVPR12:Hashing}, the authors show that the Hamming distance and code inner product is in one-to-one correspondence. That is
\begin{equation} 
\D = \frac{\L-\Z^T \Z}{2}
\label{eq:inner-hamming}
\end{equation}
where $\L$ is a matrix of all-$L$s. 

Substituting (\ref{eq:inner-hamming}) into (\ref{eq:dis-preserving}), we get the unified formulation for unsupervised and supervised hashing as 
\begin{equation}
\min_{\Z \in \{-1,1\}^{L \times n}} \norm{\Z^T\Z - \Y}^{2}
\label{eq:dis-preserving_unsup1}
\end{equation}
where $\Y = \L- \frac{L\S}{2}$ and $\Y = \L- 2L\S$ for unsupervised and supervised hashing, respectively. Note that since $\S$ is symmetric, $\Y$ is also symmetric.

The optimization problem (\ref{eq:dis-preserving_unsup1}) is non-convex and difficult to solve, i.e. NP-hard, due to the binary constraint. In order to overcome this challenge, we use the coordinate descent approach which learns one bit, i.e. one row of $\Z$,  
at a time, while keeping other rows fixed. Our coordinate descent approach for learning binary codes is shown in Algorithm~\ref{alg0}.

\begin{algorithm}[!t]
	\scriptsize
	\caption{Coordinate descent with Augmented Lagrangian (AL) / Semidefinite Relaxation (SDR)}
	\begin{algorithmic}[1] 
		\Require 
			\Statex Similarity matrix $\S$; training data $\X$; code length $L$; maximum iteration number $max\_iter$.
		\Ensure 
			\Statex Binary code matrix $\Z$.
			\Statex
			\State Initialize the binary code matrix $\Z$.
			\For{$r = 1 \to max\_iter$}
				\For {$k=1 \to L$}
				\State $\x$ $\leftarrow$ solve BQP (\ref{eq:bqp1}) for row $k$ of $\Z$ with SDR (Sec.~\ref{subsec:SDR}) or AL (Sec.~\ref{subsec:AL}).
				\State Update row $k$ of $\Z$ with $\x$. 
				\EndFor
			\EndFor
			\State Return $\Z$
    \end{algorithmic}
    \label{alg0}
\end{algorithm}

When solving for the bit $k$ (i.e. row $k$) of $\Z$, solving (\ref{eq:dis-preserving_unsup1}) is equivalent to solving the following problem 
\begin{eqnarray}
&&\min_{\z^{(k)} \in \{-1,1\}^n}  \sum_{i=1}^n \sum_{j=1}^n 2 {z_i}^{(k)} {z_j}^{(k)} \left( {\bar{\z}_i}^T \bar{\z}_j - y_{ij} \right) + const
\label{eq:bqp0}
\end{eqnarray}
where $\z^{(k)}$ is transposing of row $k$ of $\Z$; $\z_i$ is binary code of sample $i$, i.e., column $i$ of $\Z$; ${z_i}^{(k)}$ is bit $k$ of sample $i$; ${\bar{\z}_i}$ is $\z_i$ excluding bit $k$.

By removing the $const$ and letting $\x = [x_1,...,x_n]^T = {\z^{(k)}}$ (for notational simplicity), (\ref{eq:bqp0}) is equivalent to the following Binary Quadratic Problem (BQP)
\begin{eqnarray}
&&\min_{\x} {\x}^T \A \x \nonumber \\
&& s.t.\  {x_i}^2 = 1,\forall i = 1,...,n. 
\label{eq:bqp1}
\end{eqnarray}
where $\A = \{a_{ij}\} \in \R^{n \times n}$; $a_{ij} = {\bar{\z}_i}^T \bar{\z}_j - y_{ij}$.

Because $\Y$ is symmetric, $\A$ is also symmetric. The constraints in (\ref{eq:bqp1}) come from the fact that $x_i \in \{-1,1\} \Leftrightarrow {x_i}^2 = 1 $. In sections~\ref{subsec:SDR} and~\ref{subsec:AL}, we present our approaches for solving (\ref{eq:bqp1}).

\vspace{-0.3cm}
\subsection{Semidefinite Relaxation (SDR) approach}
\label{subsec:SDR}
Let us start with the following proposition
\begin{proposition}
Let matrix $\B = \A -\lambda_1 \I$, where $\lambda_1$ is the largest eigenvalue of $\A$, then
\begin{itemize}
\item (\ref{eq:bqp1}) is equivalent to 
\begin{eqnarray}
&&\min_{\x} {\x}^T \B \x \nonumber \\
&& s.t.\  {x_i}^2 = 1,\forall i = 1,...,n. 
\label{eq:bqp2} 
\end{eqnarray}
\item $\B$ is negative semidefinite.
\end{itemize}
\end{proposition}
\begin{proof}
\begin{itemize}
\item we have
\begin{eqnarray}
{\x}^T \B \x &=&  {\x}^T \A \x - {\x}^T (\lambda_1 \I) \x \nonumber \\
&=& {\x}^T \A \x - \sum_{i=1}^n \lambda_1 {x_i}^2 \nonumber \\
&=& {\x}^T \A \x - n\lambda_1
\end{eqnarray}
As $n\lambda_1$ is constant, solving (\ref{eq:bqp1}) is equivalent to solving (\ref{eq:bqp2}). $\square$
\item As $\A$ is symmetric, $\B$ is also symmetric. The symmetric matrix $\A$ can be decomposed as $\A = \U\E\U^T$, where $\E$ is diagonal matrix; $diag(\E)$ are eigenvalues of $\A$; columns of $\U$ are eigenvectors of $\A$ and $\U\U^T = \I$. We have  
\begin{eqnarray}
{\x}^T \B \x &=& {\x}^T \U\E\U^T \x - n\lambda_1 \nonumber \\
&=& \v^T \E \v - n\lambda_1 \nonumber \\
&\le & \lambda_1 \v^T\v - n\lambda_1 \nonumber \\
&=& n\lambda_1 - n\lambda_1 \nonumber \\
&=& 0
\end{eqnarray}
where $\v = \U^T\x$. The second last equation comes from the fact that 
$\v^T\v = {\x}^T \U\U^T \x = {\x}^T \x = n$. The last equation means $\B \preceq 0$\footnote{The notations $\preceq 0$ and $\succeq 0$ mean negative semidefinite and positive semidefinite, respectively.}. $\square$
\end{itemize}
\end{proof}
Because (\ref{eq:bqp1}) and (\ref{eq:bqp2}) are equivalent, we solve (\ref{eq:bqp2}), instead of (\ref{eq:bqp1}). The reason is that we will use the negative semidefinite property of $\B$ to derive the bounds on the objective value of solution of the relaxation. 
Note that, because $\B \preceq 0$, the objective function value of (\ref{eq:bqp2}) is non-positive. 

\vspace{-0.3cm}
\subsubsection{Solving}
Solving (\ref{eq:bqp2}) is challenge due to the binary constraint which is NP-hard. In this work, we rely on the semidefinite programming relaxation approach~\cite{Vandenberghe:1996:SP:233104.233107,5447068}. 
By introducing new variable, $\X = \x{\x}^T$, 
(\ref{eq:bqp2}) can be exactly rewritten as
\begin{eqnarray}
&& \min_{\X}\ trace(\B\X) \nonumber \\
&& s.t.\ diag(\X) = \mathbf{1}; \X \succeq 0; rank(\X) = 1
\label{SDR_rank1}
\end{eqnarray}
The objective function and the constraints in (\ref{SDR_rank1}) are convex in $\X$, excepting the rank one constraint. If we drop the rank one constraint, (\ref{SDR_rank1})  becomes a semidefinite program  
\begin{eqnarray}
&& \min_{\X}\ trace(\B\X) \nonumber \\
&& s.t.\ diag(\X) = \mathbf{1}; \X \succeq 0 
\label{SDR_norank}
\end{eqnarray}
We call (\ref{SDR_norank}) as semidefinite relaxation (SDR) of (\ref{eq:bqp2}). The solving of SDR problem (\ref{SDR_norank}) has been well studied. There are several widely used convex optimization packages such as SeDuMi~\cite{paper:sturm:99:sedumi}, SDPT3~\cite{Toh99sdpt3--} which use iterior-point method for solving (\ref{SDR_norank}). Because (\ref{SDR_norank}) is a convex optimization, its global optimal solution can be achieved by using the mentioned packages.

After getting the global optimal solution $\X^*$ of (\ref{SDR_norank}), the only remaining problem is how to convert $\X^*$ to a feasible solution of (\ref{eq:bqp2}). In this work, we follow the randomized rounding method proposed in~\cite{Goemans:1995:IAA:227683.227684}. Given $\X^*$, we generate vector $\mathbf{\xi}$ by  $\mathbf{\xi} \sim \mathcal{N}(0,\X^{*})$ and construct the feasible point $\hat{\x}$ of (\ref{eq:bqp2}) as 
\begin{equation}
\hat{\x} = sgn (\mathbf{\xi})
\end{equation}
This process is done multiple times, and the $\hat{\x}$ point which provides minimum objective value (of (\ref{eq:bqp2})) is selected as the solution of (\ref{eq:bqp2}).

\vspace{-0.3cm}
\subsubsection{Bounding on the objective value of SDR-rounding solution}
Let $f_{opt}$ be global optimum objective value of (\ref{eq:bqp2}) and $f_{SDR-round}$ be objective value at $\hat{\x}$ which is achieved by above rounding procedure, i.e, $f_{SDR-round} = {\hat{\x}}^T \B \hat{\x}$.  
We are interesting in finding how is $f_{SDR-round}$ close to $f_{opt}$.  
In~\cite{Goemans:1995:IAA:227683.227684,nesterov1998semidefinite}, under some conditions on the matrix $\B$, the authors derived bounds on $f_{SDR-round}$ 
 to maximization problem of the form (\ref{eq:bqp2}). In this paper, we derive bounds for the minimization problem (\ref{eq:bqp2}), where $\B \preceq 0$. The bounds on $f_{SDR-round}$ is achieved by the following theorem
\begin{theorem}
$f_{opt} \le E[f_{SDR-round}] \le \frac{2}{\pi}f_{opt}$
\label{eq:theorem-bound}
\end{theorem}

\begin{proof}
\begin{itemize}
\item Because solving SDR (\ref{SDR_norank}), following by rounding procedure, is relaxation approach to achieve a feasible solution for (\ref{eq:bqp2}), we have
\begin{equation}
f_{opt} \le f_{SDR-round}
\label{eq:opt_sdr-round}
\end{equation}  
\item  
Given $\X^*$, i.e., the global minimum solution of (\ref{SDR_norank}), let the global optimum objective value of (\ref{SDR_norank}) at $\X^*$ be $f_{SDR} = trace(\B\X^*)$; given $\hat{\x}$, i.e., the solution of (\ref{eq:bqp2}), achieved from $\X^*$ by applying the rounding procedure, in~\cite{Goemans:1995:IAA:227683.227684}, the authors show that the expected value of $f_{SDR-round}$ is 
\begin{equation}
E[f_{SDR-round}] =  E[{\hat{\x}}^T \B \hat{\x}] = \frac{2}{\pi} trace(\B arcsin(\X^*))
\label{eq:theorem_bound}
\end{equation}
where the $arcsin$ function is applied componentwise. Note that since $\X^* \succeq 0$  and $diag(\X^*)=\mathbf{1}$, the absolute value of its elements is $\le 1$. Hence $arcsin(\X^*)$ is well defined. 
 
Because $\X^* \succeq 0$, we have $arcsin(\X^*) - \X^* \succeq 0$~\cite{Ben-Tal:2001:LMC:502969}. Because $\B \preceq 0$, we have 
\begin{eqnarray}
&&trace\left(\B(arcsin(\X^*) - \X^*)\right)  \le 0 \nonumber \\
\Leftrightarrow && trace(\B arcsin(\X^*)) \le trace (\B \X^*) \nonumber \\
\Leftrightarrow && \frac{2}{\pi}trace(\B arcsin(\X^*)) \le \frac{2}{\pi}trace (\B \X^*)
 \nonumber \\
\Leftrightarrow && E[f_{SDR-round}] \le \frac{2}{\pi}trace (\B \X^*) \nonumber \\
\Leftrightarrow && E[f_{SDR-round}] \le \frac{2}{\pi}f_{SDR}
\label{eq:expected-sdr}
\end{eqnarray} 
Because (\ref{SDR_norank}) is a relaxation of (\ref{eq:bqp2}) (by removing the rank-one constraint), we have 
\begin{equation}
f_{SDR} \le f_{opt} 
\label{eq:sdr-opt}
\end{equation}
By combining (\ref{eq:expected-sdr}) and (\ref{eq:sdr-opt}), we have 
\begin{equation}
E[f_{SDR-round}] \le \frac{2}{\pi}f_{opt}
\label{eq:srd-round_opt}
\end{equation}
The proof is done by (\ref{eq:opt_sdr-round}) and (\ref{eq:srd-round_opt}). $\square$
\end{itemize}
\end{proof}

\vspace{-0.5cm}
\subsubsection{The advantages and disadvantages of SDR approach}
As mentioned, because (\ref{SDR_norank}) is a convex optimization, its global optimal solution can be achieved by using convex optimization methods. Using randomized rounding to convert SDR's solution to binary solution provides a good bound on the objective value. However, there are two main concerns, i.e., memory and computational complexity, with SDR approach. SDR approach works in the space of $n^2$ of variables, instead of $n$ as original problem.  
By using interior-point method which is traditional approach for solving SDP problem, (\ref{SDR_norank}) is solved with high complexity, i.e, $\O(n^{4.5})$\cite{5447068}. These two disadvantages may limit the capacity of SDR approach when $n$ is large. 
\vspace{-0.3cm}
\subsection{Augmented Lagrangian approach}
\label{subsec:AL}
We propose to directly solve the equality constrained minimization (\ref{eq:bqp1}) using Augmented Lagrangian (AL) method. 

\vspace{-0.3cm}
\subsubsection{Formulation}
In our formulation, we rewrite the binary constraints of (\ref{eq:bqp1}) in vector form as $\Phi(\x) = \left[(x_1)^2-1,...,(x_n)^2-1\right]^T$; let $\Lambda = [\lambda_1,...,\lambda_n]^T$ be Lagrange multipliers. By using augmented Lagrangian method, we target to minimize the following unconstrained augmented Lagrangian function
\begin{eqnarray}
\mathcal{L}(\x,\Lambda;\mu) &=&\x^T\A\x - \Lambda^T \Phi(\x) + \frac{\mu}{2}\norm{\Phi(\x)}^2
\label{eq:LA}
\end{eqnarray}
where $\mu$ is penalty parameter on the constrains. The AL algorithm for solving (\ref{eq:LA}) is presented in Algorithm \ref{alg1}. 
When $\mu$ is large, we penalize the constraint violation severely, thereby forcing the minimizer of the augmented Lagrangian function (\ref{eq:LA}) closer to the feasible region of the original constrained function (\ref{eq:bqp1}). It has been theoretically shown in~\cite{Nocedal06} that because the Lagrange multiplier $\Lambda$ is improved at every step of the algorithm, it is not necessary to take $\mu \rightarrow \infty$ in order to achieve a local optimum of (\ref{eq:bqp1}). 

\begin{algorithm}[!t]
	\scriptsize
	\caption{Augmented Lagrangian Algorithm}
	\begin{algorithmic}[1] 
		\Require 
			\Statex matrix $\A$; starting points $\x_0^s$ and $\Lambda_0$; positive numbers: $\mu_0$, $\alpha$, $\epsilon$; iteration number $T$.
		\Ensure 
			\Statex Solution $\x$
			\Statex 
			\For{$t = 0 \to T$}
			\State Find an approximate minimizer $\x_t$ of (\ref{eq:LA}), i.e., $\x_t = \underset{\x}{\arg\min}\  \mathcal{L}(\x,\Lambda_t;\mu_t)$, using $\x_t^s$ as starting point.
			\If{$t>1$ and $|{\x_t}^T\A\x_t-{\x_t^s}^T\A \x_t^s|<\epsilon$}
					\State break;
		    \EndIf
			\State Update Lagrange multiplier: $\Lambda_{t+1} = \Lambda_t - \mu_t\Phi(\x_t)$
			\State Update penalty parameter: $\mu_{t+1} = \alpha\mu_t$
			\State Set starting point for the next iteration to $\x_{t+1}^s = \x_t$
			\EndFor
			\State Return $\x_t$ 
    \end{algorithmic}
    \label{alg1}
\end{algorithm}

\vspace{-0.4cm}
\subsubsection{Complexity analysis of Augmented Lagrangian approach}
The gradient of (\ref{eq:LA}) is computed as follows
\begin{equation}
\nabla_x \mathcal{L} = 2\A\x - 2\Lambda \odot \x + 2\mu \Phi(\x) \odot \x 
\label{eq:LA_grad}
\end{equation}
where $\odot$ denotes Hadamard product.

The complexity for computing the objective function (\ref{eq:LA}) is $\O(n^2)$ and for computing the gradient (\ref{eq:LA_grad}) is also $\O(n^2)$. For finding the approximate minimizer $\x_t$ at line 2 of the Algorithm \ref{alg1}, we use LBFGS~\cite{Nocedal06_LBFGS} belonging to the family of quasi-Newton's methods. There are two main benefits with LBFGS. Firstly, the approximated Hessian matrix does not need to be explicitly computed when computing the search direction. By using two-loop recursion~\cite{Nocedal06_LBFGS}, the complexity for computing the search direction is only $\O(n)$. Hence
the computational complexity of LBFGS is $\O(t_1n^2)$, where $t_1$ is number of iterations of LBFGS. Hence, the computational complexity of Algorithm \ref{alg1} is $\O(tt_1n^2)$. In our empirical experiments, $t,t_1 \ll n$, e.g., the Algorithm \ref{alg1} converges for $t_1 \le 50$ and $t \le 10$.
Secondly, because the Hessian matrix does not need to be explicitly computed, the memory complexity of LBFGS is only $\O(n)$. Table~\ref{tab:SDR_AL} summarizes the memory and the computational complexity of SDR and AL approaches. We can see that AL approach advances SDR approach on both memory and computational complexity. 
However, the performance of AL is slightly lower than SDR. We provide detail analysis on their performance in the experimental section.
\begin{table}[!t]
\centering
\caption{Memory and computational complexity of SDR and AL}
\begin{tabular}{|c|c|c|} 
\hline
	&Computational &Memory  \\ \hline
SDR &$\O(n^{4.5})$	&$\O(n^2)$	 \\ \hline
AL  &$\O(tt_1n^2)$; $t_1 \le 50$; $t \le 10$ 	&$\O(n)$ \\	\hline
\end{tabular}
\label{tab:SDR_AL}
\vspace{-0.3cm}
\end{table}

\vspace{-0.4cm}
\subsubsection{Initialization in Augmented Lagrangian}
The Algorithm \ref{alg1} needs the initialization for $\x$ and $\Lambda$. A good initialization not only makes the algorithm robust but also leads to fast convergence. 

The initialization of $\x$ is first done by spectral relaxation, resulting the continuous solution. The continuous solution is then binarized, resulting binary solution. Specifically, we first solve (\ref{eq:bqp1}) by using spectral relaxation, i.e., 
\begin{equation}
\min_{\norm{\x}^2 = n} {\x}^T \A \x 
\label{eq:spectral}
\end{equation}
The closed-form solution of (\ref{eq:spectral}) is $\x = \sqrt{n}\u_n$, where $\u_n$ is the eigenvector corresponding to the smallest eigenvalue of $\A$.
We then optimally binarize from the first element to the last element of $\x$. When solving the binarizing for $i^{th}$ element of $\x$, we fix all remaining elements (elements 1 to $i-1$ are already binary and elements $i+1$ to $n$ are still continuous) and solve the following optimization
\begin{equation}
\min_{x_i \in \{-1,1\}} {\x}^T \A \x 
\label{eq:binary-element}
\end{equation}
By expanding and removing constant terms, (\ref{eq:binary-element}) is equivalent to 
\begin{equation}
\min_{x_i \in \{-1,1\}} x_i \left( {\bar{\x}}^T \bar{\a}_i\right) 
\label{eq:binary-element_1}
\end{equation}
where $\bar{\x}$ is vector $\x$ excluding $x_i$; $\bar{\a}_i$ is $i^{th}$ column of $\A$ excluding $i^{th}$ element. It is easy to see that the optimal solution of (\ref{eq:binary-element_1}) is $x_i = -\textrm{sgn}({\bar{\x}}^T \bar{\a}_i)$. After solving the binarizing for all elements of $\x$, the resulted binary vector is used as initialization, i.e. $\x_0^s$, in the Algorithm ~\ref{alg1}.

After getting $\x_0^s$, given $\mu_0$, we compute the corresponding $\Lambda_0$ by using the optimality condition for unconstrained minimization (\ref{eq:LA}), i.e., $\nabla_x \mathcal{L} = 0$.
By assigning (\ref{eq:LA_grad}) to zeros and using the fact that $\Phi(\x_0^s)$ equals to zeros, we have
\begin{equation}
\Lambda_0 = \left(\A\x_0^s\right) ./ \x_0^s
\end{equation}
where $./$ operator denotes element-wise division. 
\vspace{-0.4cm}
\subsection{Relationship to existing methods}
\label{subsec:diff}

In~\cite{DBLP:conf/iccv/LinSSH13,CVPR2014Lin}, the authors use two-step hashing approach for supervised hashing while our formulations are for both supervised and unsupervised hashing. 
In~\cite{DBLP:conf/iccv/LinSSH13}, when solving for row $k$ of $\Z$, i.e., $\z^{(k)}$, the authors also form the problem as a binary quadratic problem. 
To handle this NP-hard problem, the authors relax the binary constraint $\z^{(k)} \in \{-1,1\}^n$ to $\z^{(k)} \in [-1,1]^n$. The relaxed problem is then solved by bound-constrained L-BFGS method~\cite{DBLP:journals/toms/ZhuBLN97}. 
In~\cite{CVPR2014Lin}, in stead of solving for whole row $k$ of $\Z$ at a time as~\cite{DBLP:conf/iccv/LinSSH13}, the authors first split $\z^{(k)}$ into several blocks. The optimization is done for each block while keeping other blocks fixed. When solving one block, they consider the problem as a graph partition problem and use GraphCut algorithm~\cite{DBLP:journals/pami/BoykovVZ01} for finding a local optimum. 
 
Our proposed methods differ from\cite{DBLP:conf/iccv/LinSSH13,CVPR2014Lin} in solving BQP. With Augmented Lagrangian (AL) approach, we consider the original constraint, without relaxing the variables to continuous domain. With Semidefinite Relaxation (SDR) approach, in spite of removing the rank one constraint, we theoretically show that the objective value of the binary solution achieved by applying randomized rounding on SDR solution is well bounded. Note that in~\cite{DBLP:conf/iccv/LinSSH13,CVPR2014Lin}, the bounding on the objective function of their binary solution is not investigated. 

The very recent work~\cite{Mukherjee_2015_ICCV} relies on two-step hashing for unsupervised hashing. The authors introduce two approaches for inferencing binary codes which try to preserve the original distance between samples. In their work,  by considering the binary constraint on $\Z$ as $\Z \in \{0,1\}^{L \times n}$, the Hamming distance matrix $\D$ is computed as $\D = \Z^T\E^T+\E\Z-2\Z^T\Z$, where $\E$ is a matrix of all 1$s$. In the first approach, the authors use augmented Lagrangian for solving the following optimization 
\begin{eqnarray}
\ &&\min_{\Z,\Y} \norm{\S - \Y}^{2} \nonumber\\
\ &&s.t.\ \Y = \Z^T\E^T+\E\Z-2\Z^T\Z;\  \Z \in [0,1]^{L \times n}  
\label{eq:NMF-AL_1}
\end{eqnarray}
where $\S$ is original distance similarity matrix; $\Y$ is an auxiliary variable. 

In the second approach, the authors form the learning of binary code $\Z$ as a non-negative matrix factorization with additional constraints as follows
\begin{eqnarray}
&&\min_{\Z_v,\H}\norm{\S_v - \M\H\Z_v}^2 \nonumber \\
&&s.t.\ \H = \I \otimes (1-\Z_v); \ \Z_v \in [0,1]^{Ln}
\label{eq:NMF}
\end{eqnarray}
where $\S_v$ and $\Z_v$ are  vector forms of $\S$ and $\Z$, respectively; $\M$ is a constant binary matrix; $\I$ is identity matrix; $\otimes$ is Kronecker product~\cite{Mukherjee_2015_ICCV}. 

The differences between our AL, SDR approaches and two above approaches of~\cite{Mukherjee_2015_ICCV} are quite clear. 
We use the coordinate descent, i.e., solving one row of $\Z$ at a time while  the optimization in~\cite{Mukherjee_2015_ICCV} is on the space of $\Z$. This may limit their  approaches when the size of $\Z$ increases (i.e., when increasing the code length $L$ and the number of training samples $n$). 
In both their approaches, to handle the difficulty of binary constraint, they solve the relaxed problem, i.e., relaxing the constraint $\Z \in \{0,1\}^{L \times n}$ to $\Z \in [0,1]^{L \times n}$. On the other hand, our AL approach solves the constraint strictly; with SDR approach, although we remove the rank one constraint, we prove that the resulted objective value is well bounded.  
Furthermore, in their first approach~\cite{Mukherjee_2015_ICCV}, the Lagrangian function only considers the first constraint of~(\ref{eq:NMF-AL_1}), i.e., in their work, the second constraint $\Z \in [0,1]^{L \times n}$ is not considered when finding the minimizer of the augmented Lagrangian function. After solving for the minimizer of the augmented Lagrangian function, 
the resulted solution is projected onto the feasible set $\Z \in [0,1]^{L \times n}$. 
Contrary to their approach, in our augmented Lagrangian function (\ref{eq:LA}), the binary constraint is directly incorporated and encoded as $\Phi(\x)$, and is solved during the optimization. 


\vspace{-0.3cm} 
\section{Experiments}
\label{sec:exp}
\vspace{-0.1cm}
In this section we first evaluate and compare binary inference methods. We then evaluate and compare our hashing framework, i.e. using the inferred binary codes for learning hash functions,  
  to the state of the art.  
\begin{figure*}[!t]
\centering
\subfigure[CIFAR10]{
       \includegraphics[scale=0.28]{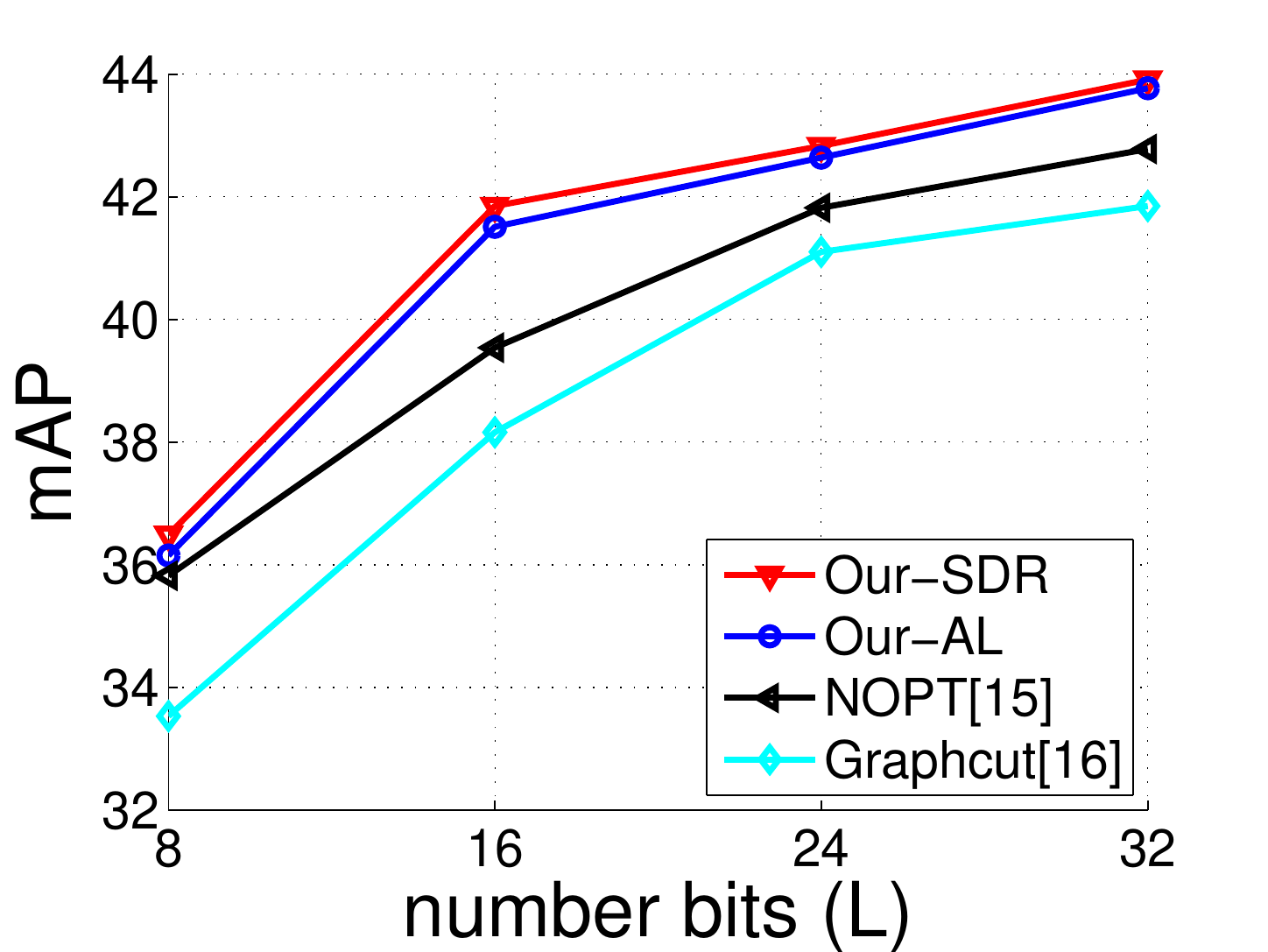}
       \label{fig:binary-infer-sup-cifar_mAP}
}\hspace{-0.5cm}
\subfigure[MNIST]{
       \includegraphics[scale=0.28]{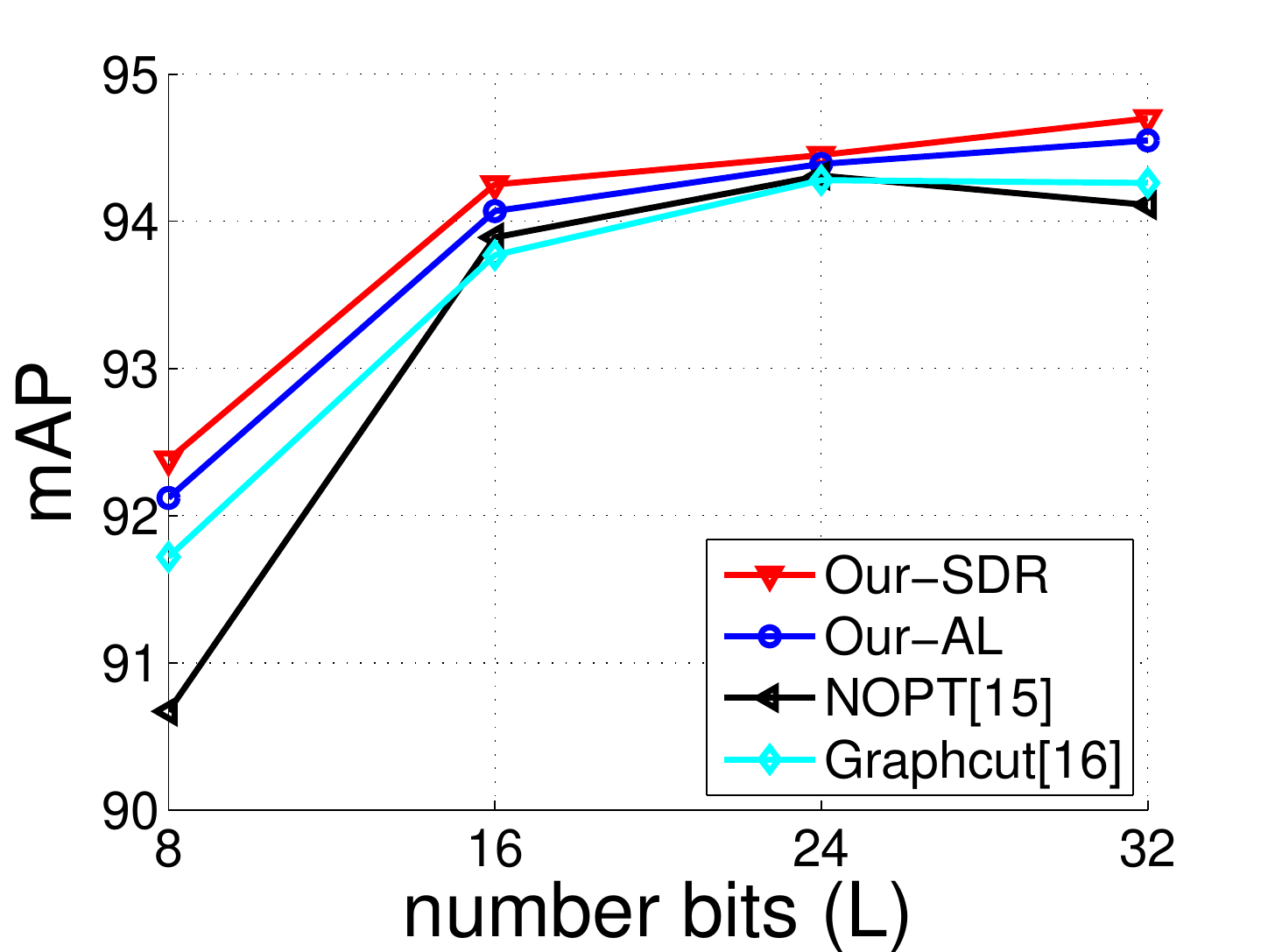} 
       \label{fig:binary-infer-sup-mnist_mAP}
}\hspace{-0.5cm}
\subfigure[SUN397]{
       \includegraphics[scale=0.28]{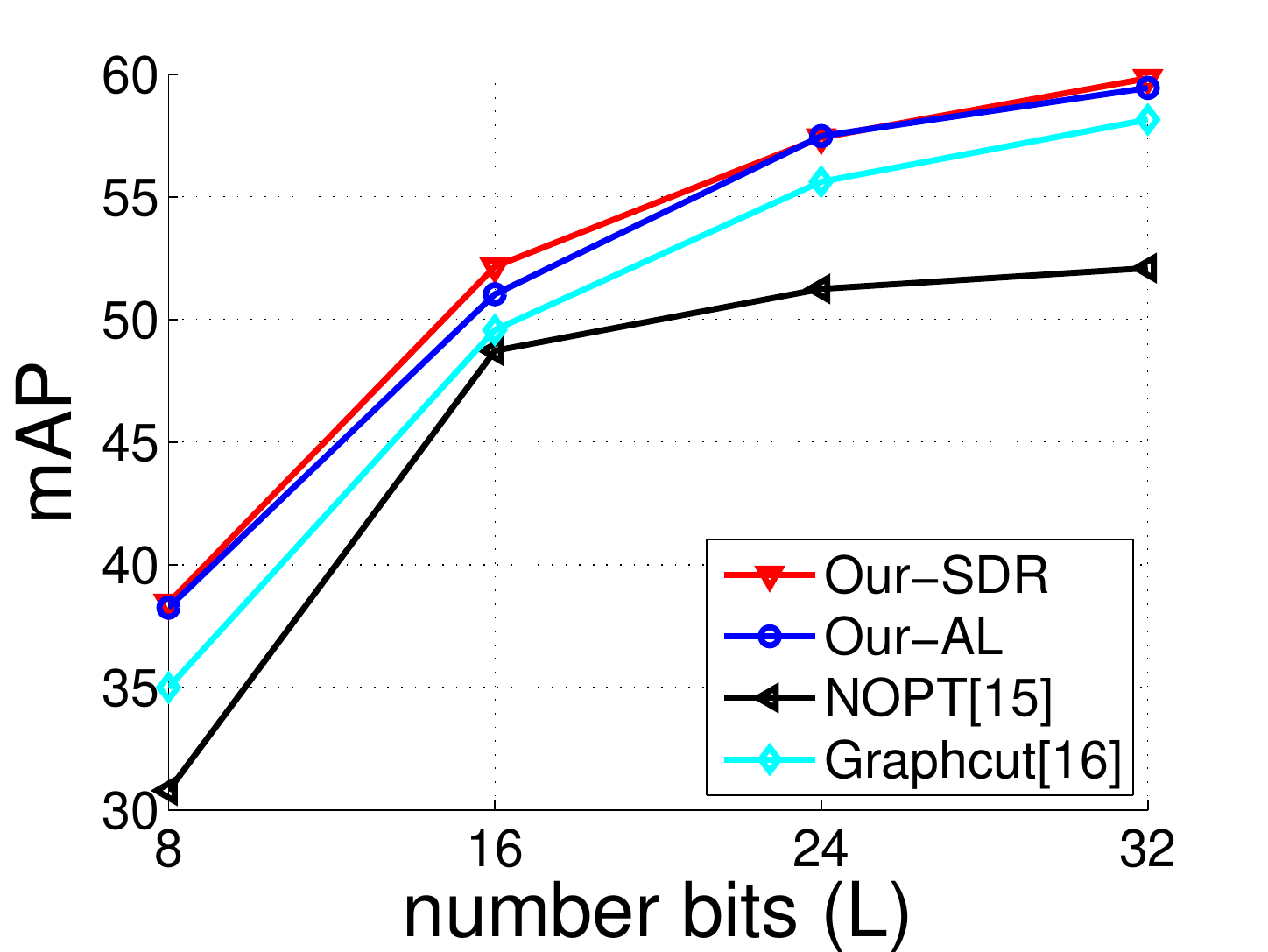}
       \label{fig:binary-infer-sup-sun397_mAP}
}
\caption[]{mAP comparison of different binary inference methods.}
\vspace{-0.2cm}
\label{fig:binary-infer-sup-cifar10-mnist-sun397-mAP}
\end{figure*}
\begin{table}[!t]
   \centering
   \footnotesize
   \caption{Precision at Hamming distance $r=2$ of different binary inference methods on CIFAR10, MNIST, SUN397.}
    \begin{tabular}{|l|c| c| c| c|c| c| c| c|c| c| c| c|}
		\hline
	  \multirow{2}{*}{} & \multicolumn{4}{|c|}{CIFAR10} & \multicolumn{4}{|c|}{MNIST} & \multicolumn{4}{|c|}{SUN397}\\
\cline{1-13}	$L$    &8 &16 &24 &32    &8 &16 &24 &32  &8 &16 &24 &32   \\ \hline 
Our-SDR						      &30.57 &46.61 &48.22 &48.43   &86.33 &93.86 &94.26 &94.56 
		&12.11 &59.19 &62.98 &61.78 \\ \hline
Our-AL	                          &30.07 &46.33 &47.95 &48.02   &85.96 &93.49 &94.09 &94.36 
		&12.03 &57.20 &63.14 &61.45 \\ \hline
NOPT\cite{DBLP:conf/iccv/LinSSH13}&29.09 &45.69 &47.41 &47.66   &81.14 &93.41 &93.88 &93.84 
		&10.08 &55.70 &60.28 &59.30 \\ \hline
Graphcut\cite{CVPR2014Lin}        &28.50 &44.64 &47.31 &47.35   &80.33 &93.31 &93.67 &93.98
		&10.70 &57.02 &61.92 &58.36 \\ \hline
	  \end{tabular}
	  \label{tab:binary-infer-sup-cifar10-mnist-sun397-pre}
	  \vspace{-0.3cm}
\end{table}
\vspace{-0.3cm}
\subsection{Dataset, implementation note, and evaluation protocol}
\subsubsection{Dataset}
CIFAR10~\cite{Krizhevsky09} dataset consists of 60,000 images of 10 classes. The training set (also used as database for retrieval)
 contains 50,000 images. The query set contains 10,000 images. Each image is represented by 320-$D$ GIST feature~\cite{gist}.

MNIST~\cite{mnistlecun} dataset consists of 70,000 handwritten digit
images of 10 classes. The training set (also used as database for retrieval)
 contains 60,000 images. The query set contains 10,000 images. Each image is represented by a 784-$D$ gray-scale feature vector by using its intensity.

SUN397~\cite{DBLP:conf/cvpr/XiaoHEOT10} contains about $108K$ images from 397
scene categories. We use a subset of this dataset including 42 categories with each containing more than 500 images (with total about $35K$ images). The query set contains 4,200 images (100 images per class) randomly sampled from the dataset. The rest images are used as database for retrieval. Each image is represented by a 4096-$D$ CNN feature produced by AlexNet~\cite{jia2014caffe}. 

For CIFAR10 and MNIST, we randomly select $500$ training samples from each class and use them for learning, i.e., using their descriptors or their labels for computing similarity matrix in unsupervised or supervised hashing. 
For SUN397, we randomly select $120$ training samples from each class for learning.

\vspace{-0.4cm}
\subsubsection{Implementation note}
After the binary code inference step with SDR/AL, the hash functions are defined by SVM with RBF kernel. The max iteration number $max\_iter$ in Algorithm~\ref{alg0} is empirically set to 3. For Augmented Lagrangian approach, its parameter in Algorithm~\ref{alg1} are empirically set by cross validation as follows: $T = 10$; $\mu_0 = 0.1$; $\alpha = 10$; $\epsilon = 10^{-6}$.

\vspace{-0.4cm}
\subsubsection{Evaluation protocol}
The ground truths of queries are defined by the class labels from the datasets. We use the following evaluation metrics which have been used in the state of the art~\cite{DBLP:conf/cvpr/GongL11,BA_CVPR15,CVPR12:Hashing,DBLP:conf/iccv/LinSSH13} to measure the performance of methods. 1) mean Average Precision (mAP); 2) precision of Hamming radius $2$ (precision$@2$) which measures precision on retrieved images having Hamming distance to query $\le 2$ (if no images satisfy, we report zero precision).  

\vspace{-0.4cm}
\subsection{Comparison between binary inference methods}
We compare our proposed methods to other binary inference methods including nonlinear optimization (NOPT) approach (i.e. using bound-constrained L-BFGS) in~\cite{DBLP:conf/iccv/LinSSH13}, Graphcut approach in~\cite{CVPR2014Lin}. 
For compared methods, we use the implementations and the suggested parameters provided by the authors.
Because the implementation of Augmented Lagrangian Method (ALM) and Nonnegative Matrix Factorization (NMF) in~\cite{Mukherjee_2015_ICCV} is not available, it is unable to compare the binary inference with those approaches. 

The proposed AL, SDR, and the compared methods require an initialization for binary code matrix $\Z$. In our work, this is the initialization at line $1$ of the Algorithm~\ref{alg0}. To make a fair comparison, we use the same initialization, i.e. the one is proposed in~\cite{Mukherjee_2015_ICCV}, for all methods. We first use PCA to project the  training matrix $\X$ from $D$ to $L$ dimensions. The projected data is then mean-thresholded, resulted binary values. After the binary code inference step, the SVM with RBF kernel is used as hash functions for all compared methods. 

Fig.~\ref{fig:binary-infer-sup-cifar10-mnist-sun397-mAP} and Table~\ref{tab:binary-infer-sup-cifar10-mnist-sun397-pre} present the mAP and the precision of Hamming radius $r$ = 2 (precision@2) of methods. In term of mAP, the proposed AL and SDR consistently outperform NOPT\cite{DBLP:conf/iccv/LinSSH13} and Graphcut\cite{CVPR2014Lin} at all code lengths. The improvement is more clear on CIFAR10 and SUN397. In term
of precision@2, the proposed AL and SDR also outperform NOPT\cite{DBLP:conf/iccv/LinSSH13} and Graphcut\cite{CVPR2014Lin}. The improvement is more clear at low code length, i.e., $L=8$. The improvement of AL and SDR over NOPT\cite{DBLP:conf/iccv/LinSSH13} and Graphcut\cite{CVPR2014Lin} means that the binary codes achieved by proposed methods are better than those achieved by NOPT\cite{DBLP:conf/iccv/LinSSH13} and Graphcut\cite{CVPR2014Lin}. 

In comparison AL and SDR, Fig.~\ref{fig:binary-infer-sup-cifar10-mnist-sun397-mAP} and Table~\ref{tab:binary-infer-sup-cifar10-mnist-sun397-pre} show SDR approach slightly outperforms AL approach. However, as analyzed in sections~\ref{subsec:SDR} and~\ref{subsec:AL}, AL approach advances SDR approach in both memory and computational complexity. 

\vspace{-0.4cm}
\subsection{Comparison with the state of the art}
We evaluate and compare the proposed SDR and AL to state-of-the-art supervised hashing methods including Binary Reconstructive Embedding (BRE)~\cite{Kulis_learningto}, ITQ-CCA~\cite{DBLP:conf/cvpr/GongL11}, KSH~\cite{CVPR12:Hashing}, Two-Step Hashing (TSH)~\cite{DBLP:conf/iccv/LinSSH13}, FashHash~\cite{CVPR2014Lin} and unsupervised hashing methods including ITQ~\cite{DBLP:conf/cvpr/GongL11}, Binary Autoencoder (BA)~\cite{BA_CVPR15}, Spherical Hashing (SPH)~\cite{CVPR12:SphericalHashing}, K-means Hashing
(KMH)~\cite{DBLP:conf/cvpr/HeWS13}. For all compared methods, we use the implementations and the suggested parameters provided by the authors.

\vspace{-0.4cm}
\subsubsection{Supervised hashing results}
\begin{figure*}[!t]
\centering
\subfigure[CIFAR10]{
       \includegraphics[scale=0.28]{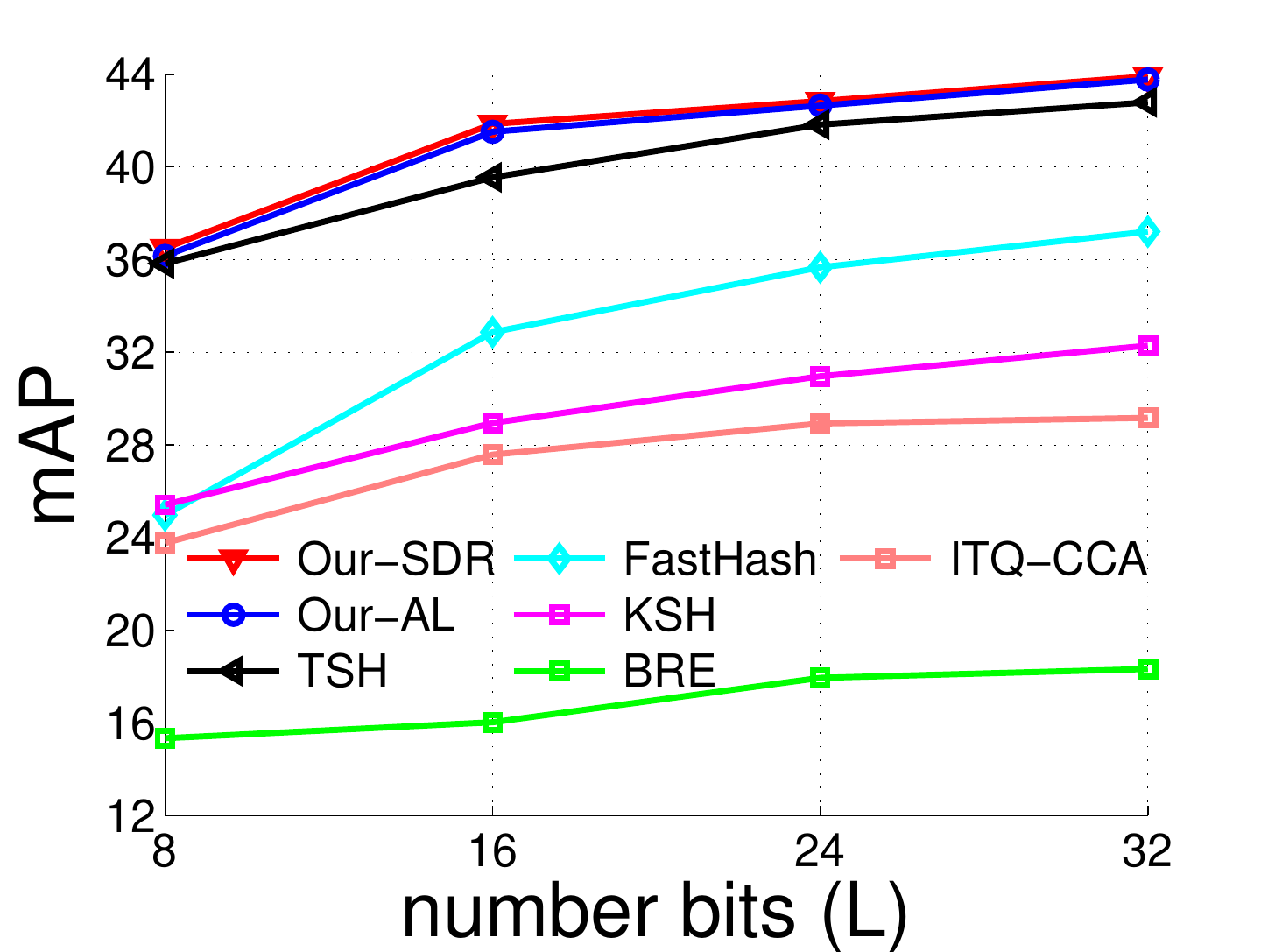}
       \label{fig:binary-infer-sup-cifar_mAP-soa}
}\hspace{-0.5cm}
\subfigure[MNIST]{
       \includegraphics[scale=0.28]{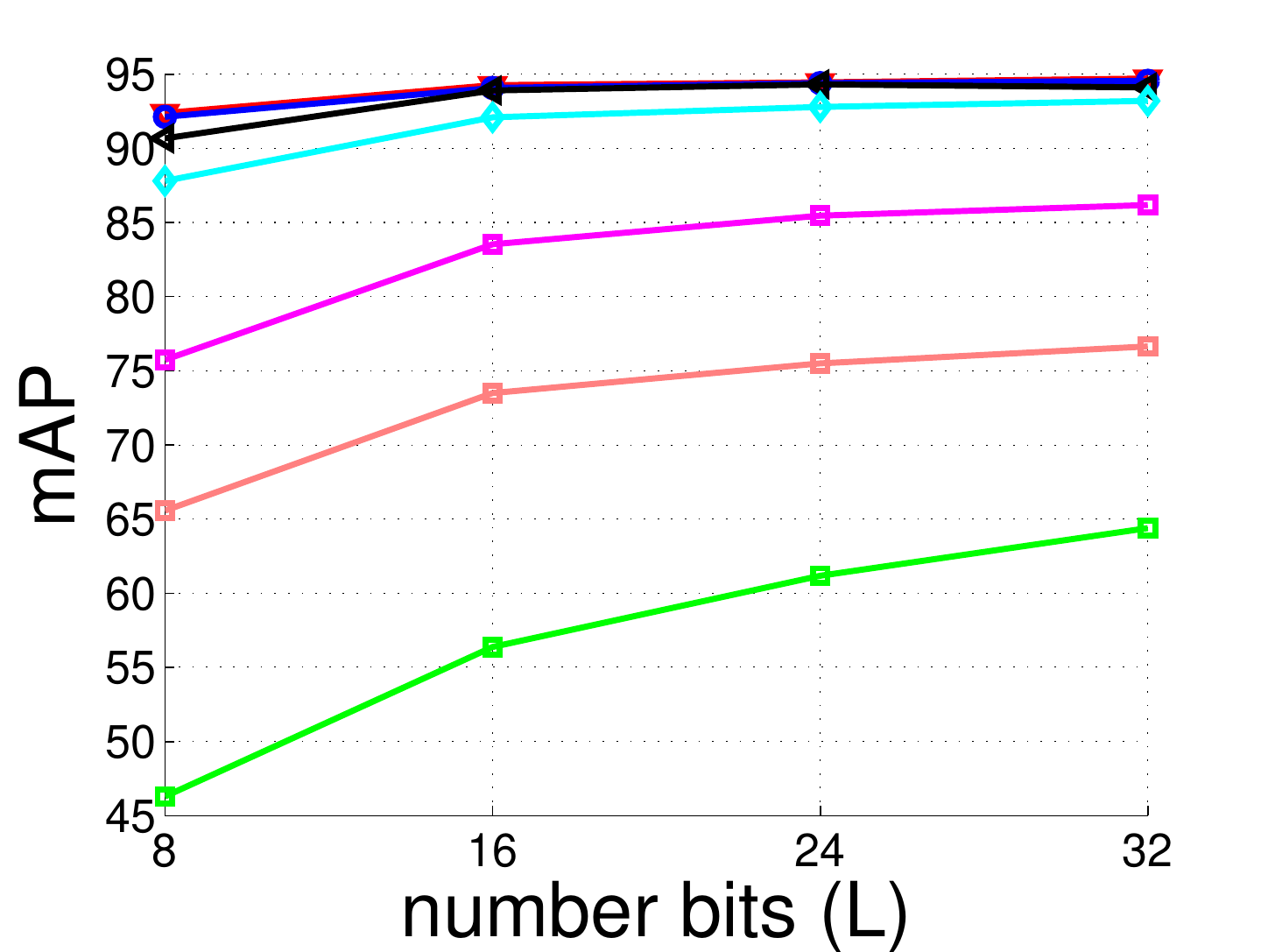} 
       \label{fig:binary-infer-sup-mnist_mAP-soa}
}\hspace{-0.5cm}
\subfigure[SUN397]{
       \includegraphics[scale=0.28]{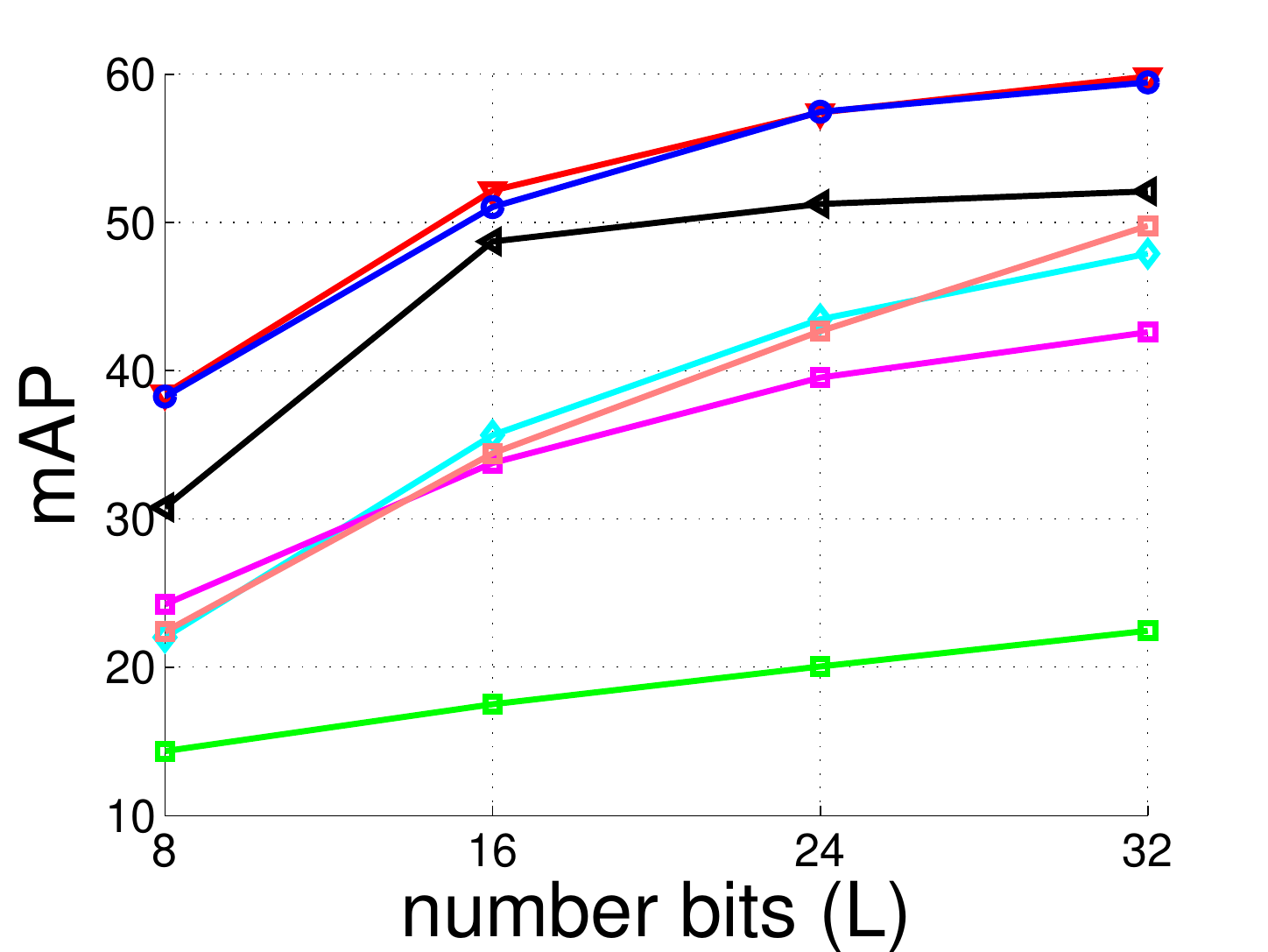}
       \label{fig:binary-infer-sup-sun397_mAP-soa}
}
\caption[]{mAP comparison with state-of-the-art supervised hashing methods.}
\label{fig:soa-sup-cifar10-mnist-sun397-mAP}
\end{figure*}

\begin{table}[!t]
   \centering
   \footnotesize
   \caption{Precision at Hamming distance $r=2$ comparison with state-of-the-art supervised hashing methods on CIFAR10, MNIST, and SUN397.}
    \begin{tabular}{|l|c| c| c| c|c| c| c| c|c| c| c| c|}
		\hline
	  \multirow{2}{*}{} & \multicolumn{4}{|c|}{CIFAR10} & \multicolumn{4}{|c|}{MNIST} & \multicolumn{4}{|c|}{SUN397}\\
\cline{1-13}	$L$    &8 &16 &24 &32    &8 &16 &24 &32  &8 &16 &24 &32   \\ \hline 
Our-SDR						      &30.57 &46.61 &48.22 &48.43   &86.33 &93.86 &94.26 &94.56 
		&12.11 &59.19 &62.98 &61.78 \\ \hline
Our-AL	                          &30.07 &46.33 &47.95 &48.02   &85.96 &93.49 &94.09 &94.36 
		&12.03 &57.20 &63.14 &61.45 \\ \hline
TSH\cite{DBLP:conf/iccv/LinSSH13} &29.09 &45.69 &47.41 &47.66   &81.14 &93.41 &93.88 &93.84 
		&10.08 &55.70 &60.28 &59.30 \\ \hline                   
FastHash\cite{CVPR2014Lin}       	&22.85 &40.81 &42.25 &32.49 &66.22 &92.14 &92.79 &91.41 
		&8.91  &46.84 &51.84 &39.40 \\ \hline                   
KSH\cite{CVPR12:Hashing}		 	&24.26 &37.26 &40.95 &36.52 &54.29 &86.94 &89.31 &88.33
		&11.79 &39.41 &51.28 &46.48 \\ \hline
BRE\cite{Kulis_learningto}       	&16.19 &22.74 &28.87 &18.41 &36.67 &70.59 &81.45 &82.83 
		&9.62  &27.93 &39.42 &30.39 \\ \hline
ITQ-CCA\cite{DBLP:conf/cvpr/GongL11}&22.66 &35.36 &38.39 &39.13 &53.46 &79.70 &82.98 &83.43
		&11.67 &36.35 &49.19 &46.81 \\ \hline
	  \end{tabular}
	  \label{tab:soa-sup-cifar10-mnist-sun397-pre}
	  \vspace{-0.3cm}
\end{table}
The mAP and precision@2 obtained by supervised hashing methods with varying code lengths are shown in Fig.~\ref{fig:soa-sup-cifar10-mnist-sun397-mAP} and Table~\ref{tab:soa-sup-cifar10-mnist-sun397-pre}, respectively. The most competitive method to AL and SDR is TSH~\cite{DBLP:conf/iccv/LinSSH13}. On CIFAR10 and MNIST datasets, the proposed AL and SDR  slightly outperform TSH while outperforming the remaining methods a fair margin. On SUN397 dataset, AL and SDR significantly outperform all compared methods. 

\vspace{-0.4cm}
\subsubsection{Unsupervised hashing results}
\begin{figure*}[!t]
\centering
\subfigure[CIFAR10]{
       \includegraphics[scale=0.28]{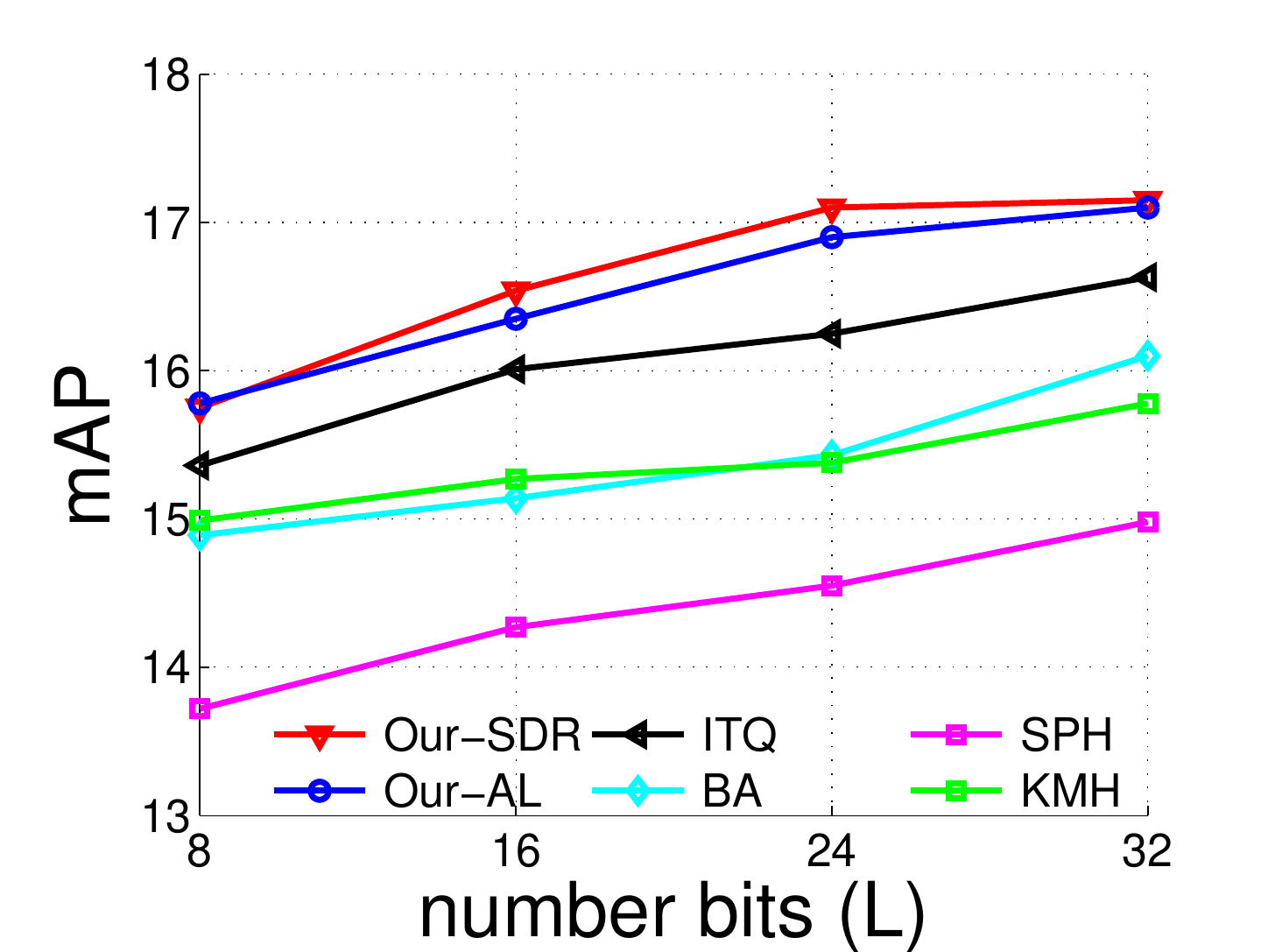}
       \label{fig:binary-infer-UNsup-cifar_mAP-soa}
}\hspace{-0.5cm}
\subfigure[MNIST]{
       \includegraphics[scale=0.28]{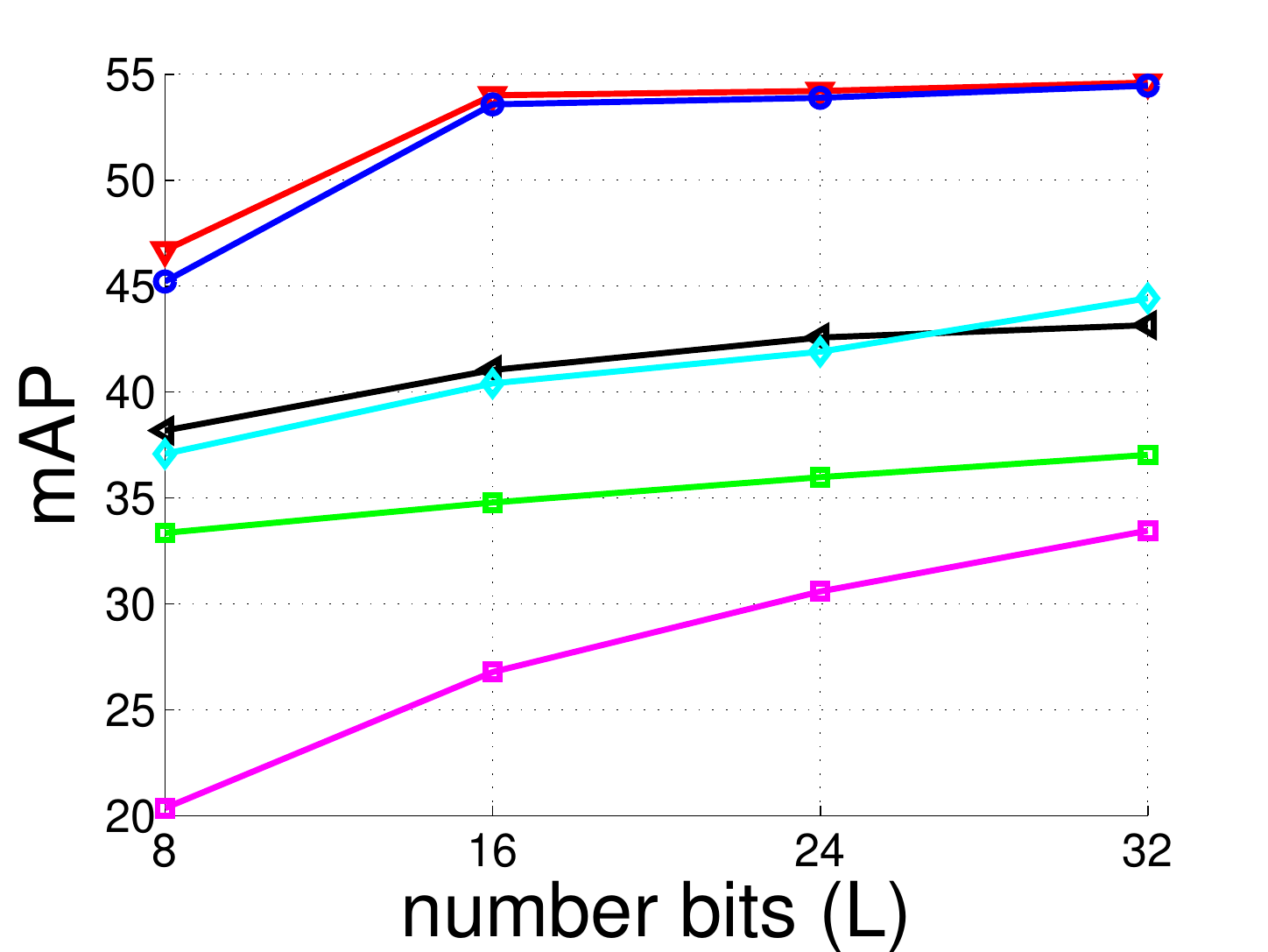} 
       \label{fig:binary-infer-UNsup-mnist_mAP-soa}
}\hspace{-0.5cm}
\subfigure[SUN397]{
       \includegraphics[scale=0.28]{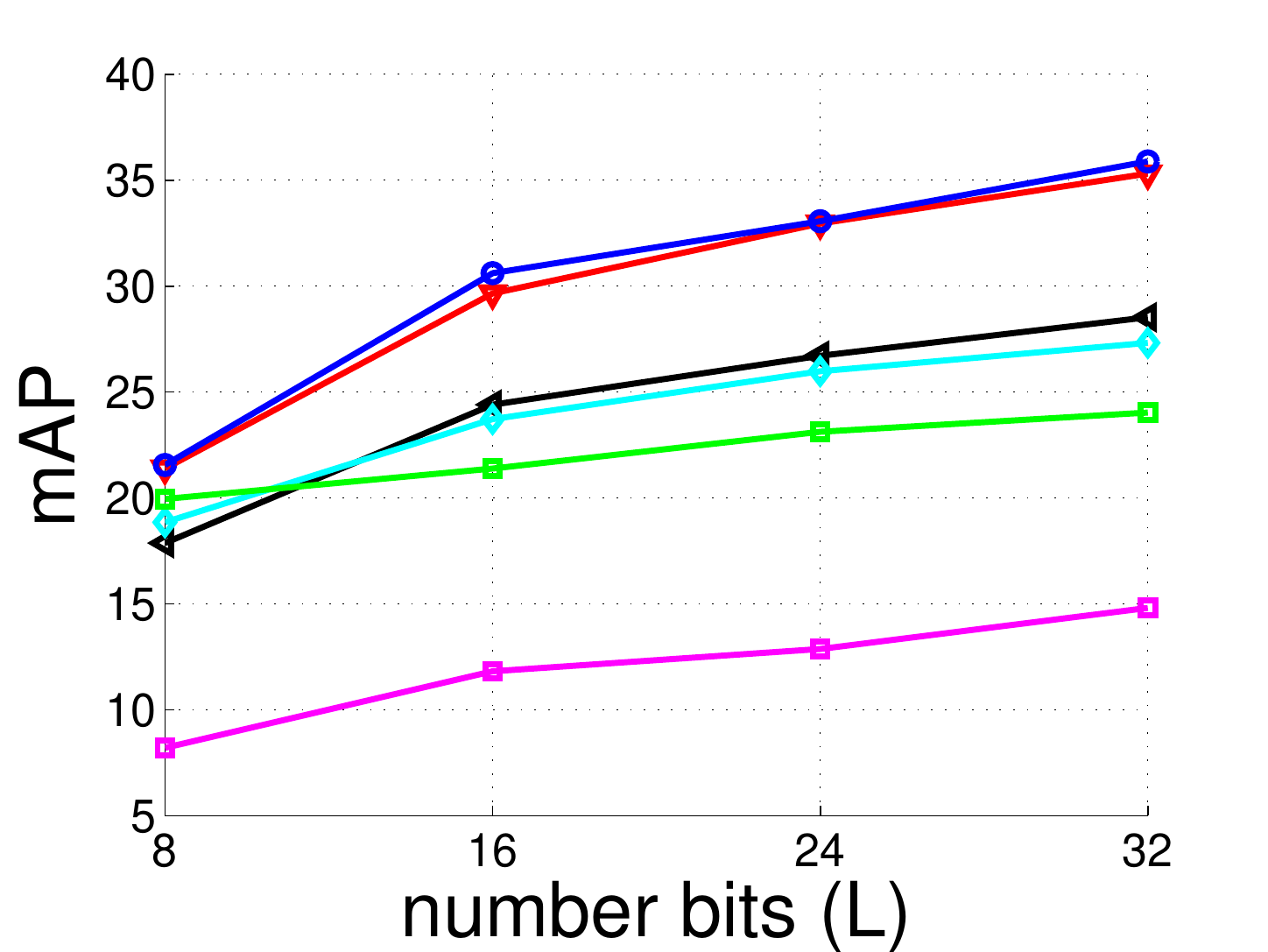}
       \label{fig:binary-infer-UNsup-sun397_mAP-soa}
}
\caption[]{mAP comparison with state-of-the-art unsupervised hashing methods.}
\label{fig:soa-UNsup-cifar10-mnist-sun397-mAP}
\vspace{-0.1cm}
\end{figure*}

\begin{table}[!t]
   \centering
   \footnotesize
   \caption{Precision at Hamming distance $r=2$ comparison with state-of-the-art unsupervised hashing methods on CIFAR10, MNIST, and SUN397.}
    \begin{tabular}{|l|c| c| c| c|c| c| c| c|c| c| c| c|}
		\hline
	  \multirow{2}{*}{} & \multicolumn{4}{|c|}{CIFAR10} & \multicolumn{4}{|c|}{MNIST} & \multicolumn{4}{|c|}{SUN397}\\
\cline{1-13}	$L$    &8 &16 &24 &32    &8 &16 &24 &32  &8 &16 &24 &32   \\ \hline 
Our-SDR	   						&17.19 &22.82 &27.40 &25.87   &43.08 &73.72 &81.34 &82.17
	&12.17 &32.15 &44.28 &45.38 \\ \hline
Our-AL	   						&17.34 &23.23 &27.26 &25.21   &42.09 &74.36 &81.50 &82.29 
    &11.99 &33.34 &44.13 &45.60 \\ \hline
ITQ\cite{DBLP:conf/cvpr/GongL11}&15.55 &22.49 &26.69 &15.36   &33.40 &69.96 &81.36 &74.70
	&9.75  &30.80 &42.07 &34.70 \\ \hline
BA\cite{BA_CVPR15}   			&15.62 &22.65 &26.55 &11.42   &32.62 &69.03 &79.11 &74.00
	&10.15 &31.61 &42.52 &31.97 \\ \hline
SPH\cite{CVPR12:SphericalHashing}&14.66 &20.32 &24.67 &12.32  &20.77 &51.74 &72.20 &63.38
	&6.38  &20.66 &30.10 &19.97 \\ \hline
KMH\cite{DBLP:conf/cvpr/HeWS13}	&15.11 &22.57 &27.25 &10.36   &32.45 &64.42 &79.97 &65.79 			&9.88  &31.04 &43.67 &28.85 \\ \hline
	  \end{tabular}
	  \label{tab:soa-UNsup-cifar10-mnist-sun397-pre}
	  \vspace{-0.1cm}
\end{table}

\begin{table}[!t]
   \vspace{-0.1em}
   \centering
   \caption{Classification accuracy on CIFAR-10 and MNIST. The results of NMF and ALM are cited from the corresponding paper~\cite{Mukherjee_2015_ICCV}.}
    \begin{tabular}{|l|c@{\hskip 0.1in}  c@{\hskip 0.1in} c@{\hskip 0.05in}|@{\hskip 0.05in} c@{\hskip 0.1in}  c@{\hskip 0.1in} c|}
		\hline
	  \multirow{2}{*}{} & \multicolumn{3}{|c|@{\hskip 0.05in}}{CIFAR10} & \multicolumn{3}{c|}{MNIST} \\
\cline{1-7}	$L$              &8     &16    &32     &8     &16     &32     \\ \hline 
Our-SDR	   					 &21.17 &24.90 &29.65  &60.68 &73.27 &81.24 \\\hline
Our-AL	   					 &21.00 &24.75 &28.84  &59.96 &72.67 &81.13 \\\hline
NMF\cite{Mukherjee_2015_ICCV}&19.77 &22.78 &23.59  &49.84 &69.65 &73.41 \\\hline
ALM\cite{Mukherjee_2015_ICCV}&19.41 &22.63 &24.27  &54.55 &69.46 &73.76 \\ \hline
	  \end{tabular}
	  \label{tab:knn}
	  \vspace{-0.4cm}
\end{table}
The mAP and precision@2 obtained by unsupervised hashing methods with varying code lengths are shown in Fig.~\ref{fig:soa-UNsup-cifar10-mnist-sun397-mAP} and Table~\ref{tab:soa-UNsup-cifar10-mnist-sun397-pre}, respectively. In term of mAP, Fig.~\ref{fig:soa-UNsup-cifar10-mnist-sun397-mAP} clearly shows that the proposed AL and SDR significantly outperform all compared methods. In term of precision@2, AL and SDR are comparable (e.g., $L=16,24$ on CIFAR10) or outperform compared methods. The improvements are more clear at high code length, i.e. $L=32$, on all datasets.
In comparison SDR and AL in unsupervised setting, two methods achieve very competitive results. 

\vspace{-0.3cm}
\paragraph{Comparison with Augmented Lagrangian Method (ALM)~\cite{Mukherjee_2015_ICCV} and Nonnegative Matrix Factorization (NMF)~\cite{Mukherjee_2015_ICCV}}: 
As the implementation of ALM~\cite{Mukherjee_2015_ICCV} and  NMF~\cite{Mukherjee_2015_ICCV} is not available, we set up the experiments on CIFAR10 and MNIST similar to~\cite{Mukherjee_2015_ICCV} to make a fair comparison. For each dataset, we randomly sample 2,000 images, 200 per class, as training set. Follow \cite{Mukherjee_2015_ICCV}, for CIFAR10, each image is represented by 625-$D$ HOG descriptors~\cite{DBLP:conf/cvpr/DalalT05}. The hash functions are defined as linear SVM. 
Similar to~\cite{Mukherjee_2015_ICCV}, we report the classification accuracy by using $k$-NN ($k=4$) classifier at varying code lengths. 
The comparative results, presented in Table~\ref{tab:knn}, clearly show that the proposed AL and SDR outperform ALM~\cite{Mukherjee_2015_ICCV} and NMF~\cite{Mukherjee_2015_ICCV} on both datasets. Although both our AL and ALM~\cite{Mukherjee_2015_ICCV} use augmented Lagrangian approach, the improvement of our AL over ALM~\cite{Mukherjee_2015_ICCV} confirms the benefit of the integration of binary constraint in the augmented Lagrangian function also the effectiveness of the proposed initialization.  

\vspace{-0.4cm}
\section{Conclusion}
\label{sec:concl}
\vspace{-0.3cm}
This paper proposes effective solutions to binary code inference step in two-step hashing where the goal is to preserve the original similarity matrix via Hamming distance in Hamming space. We cast the learning of one bit code as the binary quadratic problem.  We propose two approaches:  Semidefinite Relaxation (SDR) and Augmented Lagrangian (AL) for solving. Extensive experiments show that both AL and SDR approaches compare favorably with the state of the art.


\bibliographystyle{splncs}
\bibliography{hash}
\end{document}